\documentclass[11pt,a4paper]{article}
\usepackage[hyperref]{acl2018}
\usepackage{times}
\usepackage{latexsym}
\usepackage{amssymb}
\usepackage{bm}
\usepackage{amsmath}
\usepackage{amsfonts}
\usepackage{amsthm}
\usepackage{latexsym}
\usepackage{booktabs}
\usepackage{siunitx}
\usepackage{xcolor}
\usepackage{dsfont}
\usepackage{graphicx}
\usepackage{caption}
\usepackage{subcaption}

\usepackage{url}

\aclfinalcopy 

\renewcommand{\thefootnote}{*}

\newcommand{\comment}[1]{}

\DeclareMathOperator*{\cov}{\text{cov}}
\DeclareMathOperator*{\corr}{\text{corr}}

\newtheorem{definition}{Definition}

\newtheorem*{claim}{Claim}
\newtheorem*{conjecture}{Conjecture}

\DeclareMathOperator{\diag}{diag}

\title{Probabilistic Embedding of Knowledge Graphs with\\ Box Lattice Measures}

\author{Luke Vilnis$^*$ \qquad Xiang Li$^*$ \qquad Shikhar Murty \qquad Andrew McCallum \\
  College of Information and Computer Sciences \\
  University of Massachusetts Amherst\\
  \texttt{\{luke,xiangl,smurty,mccallum\}@cs.umass.edu} }

\date{}

\begin{document}
\maketitle

\let\thefootnote\relax\footnote{\textsuperscript{*} Equal contribution.}

\begin{abstract}
Embedding methods which enforce a partial order or lattice structure over the concept space, such as Order Embeddings (OE) \cite{orderembedding}, are a natural way to model transitive relational data (e.g. entailment graphs). However, OE learns a deterministic knowledge base, limiting expressiveness of queries and the ability to use uncertainty  for both prediction and learning (e.g. learning from expectations). Probabilistic extensions of OE \cite{lai2017learning} have provided the ability to somewhat calibrate these \emph{denotational probabilities} while retaining the consistency and inductive bias of ordered models, but lack the ability to model the negative correlations found in real-world knowledge. In this work we show that a broad class of models that assign probability measures to OE can never capture negative correlation, which motivates our construction of a novel \emph{box lattice} and accompanying probability measure to capture anticorrelation and even disjoint concepts, while still providing the benefits of probabilistic modeling, such as the ability to perform rich joint and conditional queries over arbitrary sets of concepts, and both learning from and predicting calibrated uncertainty. We show improvements over previous approaches in modeling the Flickr and WordNet entailment graphs, and investigate the power of the model.
\end{abstract}

\section{Introduction}

Structured embeddings based on regions, densities, and orderings have gained popularity in recent years for their inductive bias towards the essential asymmetries inherent in problems such as image captioning \cite{orderembedding}, lexical and textual entailment \cite{Erk:2009:RWR:1596374.1596387,vilnis2014word,lai2017learning,athiwaratkun2018on}, and knowledge graph completion and reasoning \cite{He:2015:LRK:2806416.2806502,NIPS2017_7213,li2017improved}.

Models that easily encode asymmetry, and related properties such as transitivity (the two components of commonplace relations such as partially ordered sets and lattices), have great utility in these applications, leaving less to be learned from the data than arbitrary relational models. At their best, they resemble a hybrid between embedding models and structured prediction. As noted by \citet{orderembedding} and \citet{li2017improved}, while the models learn sets of embeddings, these parameters obey rich structural constraints. The entire set can be thought of as one, sometimes provably consistent, structured prediction, such as an ontology in the form of a single directed acyclic graph.

While the structured prediction analogy applies best to Order Embeddings (OE), which embeds consistent partial orders, other region- and density-based representations have been proposed for the express purpose of inducing a bias towards asymmetric relationships. For example, the Gaussian Embedding (GE) model \cite{vilnis2014word} aims to represent the asymmetry and uncertainty in an object's relations and attributes by means of uncertainty in the representation. However, while the space of representations is a manifold of probability distributions, the model is not truly probabilistic in that it does not model asymmetries and relations in terms of probabilities, but in terms of asymmetric comparison functions such as the originally proposed KL divergence and the recently proposed \emph{thresholded divergences} \cite{athiwaratkun2018on}.

Probabilistic models are especially compelling for modeling ontologies, entailment graphs, and knowledge graphs. Their desirable properties include an ability to remain consistent in the presence of noisy data, suitability towards semi-supervised training using the expectations and uncertain labels present in these large-scale applications, the naturality of representing the inherent uncertainty of knowledge they store, and the ability to answer complex queries involving more than 2 variables. Note that the final one requires a true joint probabilistic model with a tractable inference procedure, not something provided by e.g. matrix factorization.

We take the dual approach to density-based embeddings and model uncertainty about relationships and attributes as explicitly probabilistic, while basing the probability on a latent space of geometric objects that obey natural structural biases for modeling transitive, asymmetric relations. The most similar work are the probabilistic order embeddings (POE) of Lai \cite{lai2017learning}, which apply a probability measure to each order embedding's forward cone (the set of points greater than the embedding in each dimension), assigning a finite and normalized volume to the unbounded space. However, POE suffers severe limitations as a probabilistic model, including an inability to model negative correlations between concepts, which motivates the construction of our box lattice model.

Our model represents objects, concepts, and events as high-dimensional products-of-intervals (\emph{hyperrectangles} or \emph{boxes}), with an event's unary probability coming from the box volume and joint probabilities coming from overlaps. This contrasts with POE's approach of defining events as the forward cones of vectors, extending to infinity, integrated under a probability measure that assigns them finite volume.

One desirable property of a structured representation for ordered data, originally noted in \cite{orderembedding} is a ``slackness'' shared by OE, POE, and our model: when the model predicts an ``edge'' or lack thereof (i.e. $P(a|b)=0$ or $1$, or a zero constraint violation in the case of OE), being exposed to that fact again will not update the model. Moreover, there are large degrees of freedom in parameter space that exhibit this slackness, giving the model the ability to embed complex structure with 0 loss when compared to models based on symmetric inner products or distances between embeddings, e.g. bilinear GLMs \cite{collins2002generalization}, Trans-E \cite{bordes2013translating}, and other embedding models which must always be pushing and pulling parameters towards and away from each other.

Our experiments demonstrate the power of our approach to probabilistic ordering-biased relational modeling. First, we investigate an instructive 2-dimensional toy dataset that both demonstrates the way the model self organizes its box event space, and enables sensible answers to queries involving arbitrary numbers of variables, despite being trained on only pairwise data. We achieve a new state of the art in denotational probability modeling on the Flickr entailment dataset \cite{lai2017learning}, and a matching state-of-the-art on WordNet hypernymy \cite{orderembedding,miller1995wordnet} with the concurrent work on thresholded Gaussian embedding of \citet{athiwaratkun2018on}, achieving our best results by training on additional co-occurrence expectations aggregated from leaf types. 

We find that the strong empirical performance of probabilistic ordering models, and our box lattice model in particular, and their endowment of new forms of training and querying, make them a promising avenue for future research in representing structured knowledge.

\section{Related Work}

In addition to the related work in structured embeddings mentioned in the introduction, our focus on directed, transitive relational modeling and ontology induction shares much with the rich field of directed graphical models and causal modeling \cite{pearl1988probabilistic}, as well as learning the structure of those models \cite{heckerman1995learning}. Work in undirected structure learning such the Graphical Lasso \cite{friedman2008sparse} is also relevant due to our desire to learn from pairwise joint/conditional probabilities and moment matrices, which are closely related in the setting of discrete variables.

Especially relevant research in Bayesian networks are applications towards learning taxonomic structure of relational data \cite{bansal2014structured}, although this work is often restricted towards tree-shaped ontologies, which allow efficient inference by Chu-Liu-Edmonds' algorithm \cite{chuliu}, while we focus on arbitrary DAGs.

As our model is based on populating a latent ``event space'' into boxes (products of intervals), it is especially reminiscent of the Mondrian process \cite{roy2009mondrian}. However, the Mondrian process partitions the space as a high dimensional tree (a non-parametric kd-tree), while our model allows the arbitrary box placement required for DAG structure, and is much more tractable in high dimensions compared to the Mondrian's Bayesian non-parametric inference.

Embedding applications to relational learning constitute a huge field to which it is impossible to do justice, but one general difference between our approaches is that e.g. a matrix factorization model treats the embeddings as objects to score relation links with, as opposed to POE or our model in which embeddings represent subsets of probabilistic event space which are directly integrated. They are full probabilistic models of the joint set of variables, rather than embedding-based approximations of only low-order joint and conditional probabilities. That is, any set of our parameters can answer any arbitrary probabilistic question (possibly requiring intractable computation), rather than being fixed to modeling only certain subsets of the joint.

Embedding-based learning's large advantage over the combinatorial structure learning presented by classical PGM approaches is its applicability to large-scale probability distributions containing hundreds of thousands of events or more, as in both our WordNet and Flickr experiments.

\section{Background}

\subsection{Partial Orders and Lattices}

A non-strict \emph{partial ordered set} (\emph{poset}) is a set $P$ equipped with a binary relation $\preceq$ such that for all $a,b,c \in P$,
\begin{itemize}  
\item $a \preceq a$ (reflexivity)
\item $a \preceq b \preceq a$ implies $a=b$ (antisymmetry)
\item $a \preceq b \preceq c$ implies $a \preceq c$ (transitivity)
\end{itemize}
This is simply a generalization of a totally ordered set that allows some elements to be incomparable, and is a good model for the kind of acyclic directed graph data found in knowledge bases.

A \emph{lattice} is a poset where any subset has a a unique least upper and greatest lower bound, which will be true of all posets (lattices) considered in this paper. The least upper bound of two elements $a,b \in P$ is called the \emph{join}, denoted $a \vee b$, and the greatest lower bound is called the \emph{meet}, denoted $a \wedge b$. 

Additionally, in a \emph{bounded lattice} we have two extra elements, called \emph{top}, denoted $\top$ and \emph{bottom}, denoted $\bot$, which are respectively the least upper bound and greatest lower bound of the entire space. Using the extended real number line (adding points at infinity), all lattices considered in this paper are bounded lattices.

\subsection{Order Embeddings (OE)}

\citet{orderembedding} introduced a method for embedding partially ordered sets and a task, \emph{partial order completion}, an abstract term for things like hypernym or entailment prediction (learning transitive relations). The goal is to learn a mapping from the partially-ordered data domain to some other partially-ordered space that will enable generalization.

\begin{definition} \citet{orderembedding}\\A function $f: (X,\preceq_X) \to (Y,\preceq _Y)$ is an \emph{order-embedding} if for all $u,v \in X$
\begin{align*}
u\preceq_X v \iff f(u) \preceq_Y f(v)
\end{align*}
\end{definition}

They choose $Y$ to be a vector space, and the order $\preceq_Y$ to be based on the \emph{reverse product order} on $\mathds{R}^n_+$, which specifies 
\begin{align*}
x \preceq y \iff \forall i \in \{1..n\},~~ x_i \ge y_i
\end{align*}
so an embedding is below another in the hierarchy if all of the coordinates are larger, and 0 provides a top element.

Although \citet{orderembedding} do not explicitly discuss it, their model does not just capture partial orderings, but is a standard construction of a vector (Hilbert) lattice, in which the operations of meet and join can be defined as taking the pointwise maximum and minimum of two vectors, respectively \cite{rieszspace}. This observation is also used in \cite{li2017improved} to generate extra constraints for training order embeddings.

As noted in the original work, these single point embeddings can be thought of as regions, i.e. the cone extending out from the vector towards infinity. All concepts ``entailed'' by a given concept must lie in this cone.

This ordering is optimized from examples of ordered elements and negative samples via a max-margin loss.

\subsection{Probabilistic Order Embeddings (POE)}

\citet{lai2017learning} built on the ``region'' idea to derive a probabilistic formulation (which we will refer to as POE) to model entailment probabilities in a consistent, hierarchical way. 

Noting that all of OE's regions obviously have the same infinite area under the standard (Lebesgue) measure of $\mathds{R}^n_+$, they propose a probabilistic interpretation where the Bernoulli probability of each concept $a$ or joint set of concepts $\{a,b\}$ with corresponding vectors $\{x,y\}$ is given by its volume under the exponential measure:
\begin{align*}
p(a)&=\exp(-\sum_i x_i)=\int\displaylimits_{z\preceq x} \exp(-\|z\|_1)dz\\
p(a,b)&=p(x\wedge y)=\exp(-\|\max(x_i,y_i)\|_1)
\end{align*}
since the meet of two vectors is simply the intersection of their area cones, and replacing sums with $\ell_1$ norms for brevity since all coordinates are positive. While having the intuition of measuring the areas of cones, this also automatically gives a valid probability distribution over concepts since this is just the product likelihood under a coordinatewise exponential distribution.

However, they note a deficiency of their model --- it can only model positive (Pearson) correlations between concepts (Bernoulli variables).

Consider two Bernoulli variables $a$ and $b$, whose probabilities correspond to the areas of cones $x$ and $y$. Recall the Bernoulli covariance formula (we will deal with covariances instead of correlations when convenient, since they always have the same sign):
\begin{align*}
&\cov(a,b)=p(a,b)-p(a)p(b)=\\
&~~\exp(-\|\max(x_i,y_i)\|_1)-\exp(-\|x_i+y_i\|_1)
\end{align*}
Since the sum of two positive vectors can only be greater than the sum of their pointwise maximum, this quantity will always be nonnegative. This has real consequences for probabilistic modeling in KBs: conditioning on more concepts will only make probabilities higher (or unchanged), e.g. $p(\text{dog}|\text{plant}) \ge p(\text{dog})$.

\subsection{Probabilistic Asymmetric Transitive Relations}

Probabilistic models have pleasing consistency properties for modeling asymmetric transitive relations, in particular compared to density-based embeddings --- a pairwise conditional probability table can almost always (in the technical sense) be asymmetrized to produce a DAG by simply taking an edge if $P(a|b)>P(b|a)$. A matrix of pairwise Gaussian KL divergences cannot be consistently asymmetrized in this manner. These claims are proven in Appendix \ref{asym-section}. While a high $P(a|b)$ does not always indicate an edge in an ontology due to confounding variables, existing graphical model structure learning methods can be used to further prune on the base graph without adding a cycle, such as Graphical Lasso or simple thresholding \cite{fattahi2017graphical}.

\section{Method}

We develop a probabilistic model for lattices based on hypercube embeddings that can model both positive and negative correlations. Before describing this, we first motivate our choice to abandon OE/POE type cone-based models for this purpose.

\subsection{Correlations from Cone Measures}
\begin{claim}
For a pair of Bernoulli variables $p(a)$ and $p(b)$, $\cov(a,b) \ge 0$ if the Bernoulli probabilities come from the volume of a cone as measured under any product (coordinate-wise) probability measure $p(x)=\prod_i^n p_i(x_i)$ on $\mathds{R}^n$, where $F_i$, the associated CDF for $p_i$, is monotone increasing.
\end{claim}
\begin{proof}
For any product measure we have
\begin{align*}
\int\displaylimits_{z\preceq x} p(z)dz =\prod_i^n \int\displaylimits_{x_i \le z_i} p_i(z_i)dz_i = \prod_i^n 1 - F_i(x_i)
\end{align*}
This is just the area of the unique box corresponding to $\prod_i^n [F_i(x_i), 1] \in [0,1]^n$, under the uniform measure. This box is unique as a monotone increasing univariate CDF is bijective with $(0,1)$ --- cones in $\mathds{R}^n$ can be invertibly mapped to boxes of equivalent measure inside the unit hypercube $[0,1]^n$. These boxes have only half their degrees of freedom, as they have the form $[F_i(x_i),1]$ per dimension, (intuitively, they have one end "stuck at infinity" since the cone integrates to infinity. 

So W.L.O.G. we can consider two transformed cones $x$ and $y$ corresponding to our Bernoulli variables $a$ and $b$, and letting $F_i(x_i)=u_i$ and $F_i(y_i)=v_i$, their intersection in the unit hypercube is $\prod_i^n [\max(u_i,v_i),1]$. 

Pairing terms in the right-hand product, we have
\begin{align*}
&p(a,b)-p(a)p(b)=\\
&\prod_i^n ( 1 - \max(u_i,v_i)) -
\prod_i^n (1 - u_i)(1-v_i) \ge 0
\end{align*}
since the right contains all the terms of the left and can only grow smaller. This argument is easily modified to the case of the nonnegative orthant, \emph{mutatis mutandis}.
\end{proof}

An open question for future work is what non-product measures this claim also applies to. Note that some non-product measures, such as multivariate Gaussian, can be transformed into product measures easily (\emph{whitening}) and the above proof would still apply. It seems probable that some measures, nonlinearly entangled across dimensions, could encode negative correlations in cone volumes. However, it is not generally tractable to integrate high-dimensional cones under arbitrary non-product measures.

\subsection{Box Lattices}
\label{sec:box-lattice}

The above proof gives us intuition about the possible form of a better representation. Cones can be mapped into boxes within the unit hypercube while preserving their measure, and the lack of negative correlation seems to come from the fact that they always have an overly-large intersection due to ``pinning'' the maximum in each dimension to 1. To remedy this, we propose to learn representations in the space of \emph{all} boxes (axis-aligned hyperrectangles), gaining back an extra degree of freedom. These representations can be learned with a suitable probability measure in $\mathds{R}^n$, the nonnegative orthant $\mathds{R}_+^n$, or directly in the unit hypercube with the uniform measure, which we elect.

We associate each concept with 2 vectors, the minimum and maximum value of the box at each dimension. Practically for numerical reasons these are stored as a minimum, a positive offset plus an $\epsilon$ term to prevent boxes from becoming too small and underflowing.

Let us define our box embeddings as a pair of vectors in $[0,1]^n$, $(x_m,x_M)$, representing the maximum and minimum at each coordinate.

Then we can define a partial ordering by inclusion of boxes, and a lattice structure as
\begin{align*}
x \wedge y &= \bot~~\text{if $x$ and $y$ disjoint, else}\\
x \wedge y &=\prod_i [\max(x_{m,i},y_{m,i}), \min(x_{M,i},y_{M,i})]\\
x \vee y &= \prod_i [\min(x_{m,i},y_{m,i}), \max(x_{M,i},y_{M,i})]
\end{align*}
where the meet is the intersecting box, or bottom (the empty set) where no intersection exists, and join is the smallest enclosing box. This lattice, considered on its own terms as a non-probabilistic object, is strictly more general than the order embedding lattice in any dimension, which is proven in Appendix \ref{lattice-properties}.

However, the finite sizes of all the lattice elements lead to a natural probabilistic interpretation under the uniform measure. Joint and marginal probabilities are given by the volume of the (intersection) box. For concept $a$ with associated box $(x_m,x_M)$, probability is simply $p(a)=\prod_i^n(x_{M,i}-x_{m,i})$ (under the uniform measure). $p(\bot)$ is of course zero since no probability mass is assigned to the empty set.

It remains to show that this representation can represent both positive and negative correlations.

\begin{claim}
For a pair of Bernoulli variables $p(a)$ and $p(b)$, $\corr(a,b)$ can take on any value in $[-1,1]$ if the probabilities come from the volume of associated boxes in $[0,1]^n$.
\end{claim}
\begin{proof}
Boxes can clearly model disjointness (exactly $-1$ correlation if the total volume of the boxes equals 1). Two identical boxes give their concepts exactly correlation 1. The area of the meet is continuous with respect to translations of intersecting boxes, and all other terms in correlation stay constant, so by continuity of the correlation function our model can achieve all possible correlations for a pair of variables.
\end{proof}
This proof can be extended to boxes in $\mathds{R}^n$ with product measures by the previous reduction.

\textbf{Limitations:} Note that this model cannot perfectly describe all possible probability distributions or concepts as embedded objects. For example, the complement of a box is not a box. However, queries about complemented variables can be calculated by the Inclusion-Exclusion principle, made more efficient by the fact that all non-negated terms can be grouped and calculated exactly. We show some toy exact calculations with negated variables in Appendix \ref{negated-variables}. Also, note that in a knowledge graph often true complements are not required --- for example \emph{mortal} and \emph{immortal} are not actually complements, because the concept $\emph{color}$ is neither.

Additionally, requiring the total probability mass covered by boxes to equal 1, or exactly matching marginal box probabilities while modeling all correlations is a difficult box-packing-type problem and not generally possible. Modeling limitations aside, the union of boxes having mass $<1$ can be seen as an open-world assumption on our KB (not all points in space have corresponding concepts, yet).

\subsection{Learning}

While inference (calculation of pairwise joint, unary marginal, and pairwise conditional probabilities) is quite straightforward by taking intersections of boxes and computing volumes (and their ratios), learning does not appear easy at first glance. While the (sub)gradient of the joint probability is well defined when boxes intersect, it is non-differentiable otherwise. Instead we optimize a lower bound.

Clearly $p(a\vee b)\ge p(a \cup b)$, with equality only when $a=b$, so this can give us a lower bound:
\begin{align*}
p(a \wedge b) &= p(a) + p(b) - p(a \cup b)\\
&\ge p(a) + p(b) - p(a \vee b)
\end{align*}
Where probabilities are always given by the volume of the associated box. This lower bound always exists and is differentiable, even when the joint is not. It is guaranteed to be nonpositive except when $a$ and $b$ intersect, in which case the true joint likelihood should be used.

While a negative bound on a probability is odd, inspecting the bound we see that its gradient will push the enclosing box to be smaller, while increasing areas of the individual boxes, until they intersect, which is a sensible learning strategy.

Since we are working with small probabilities it is advisable to negate this term and maximize the negative logarithm:
\begin{align*}
-\log(p(a \vee b) - p(a) - p(b))
\end{align*}
This still has an unbounded gradient as the lower bound approaches 0, so it is also useful to add a constant within the logarithm function to avoid numerical problems.

Since the likelihood of the full data is usually intractable to compute as a conjunction of many negations, we optimize binary conditional and unary marginal terms separately by maximum likelihood.

In this work, we parametrize the boxes as $(\text{min},\Delta=\text{max}-\text{min})$, with Euclidean projections after gradient steps to keep our parameters in the unit hypercube and maintain the minimum/delta constraints.

Now that we have the ability to compute probabilities and (surrogate) gradients for arbitrary marginals in the model, and by extension conditionals, we will see specific examples in the experiments.

\section{Experiments}

\subsection{Warmup: 2D Embedding of a Toy Lattice}
\label{warmup}

{\centering
{\centering
\begin{figure*}[!htbp]
\captionsetup[subfigure]{position=b}
\centering
\subcaptionbox{Original lattice}{\includegraphics[width=24em]{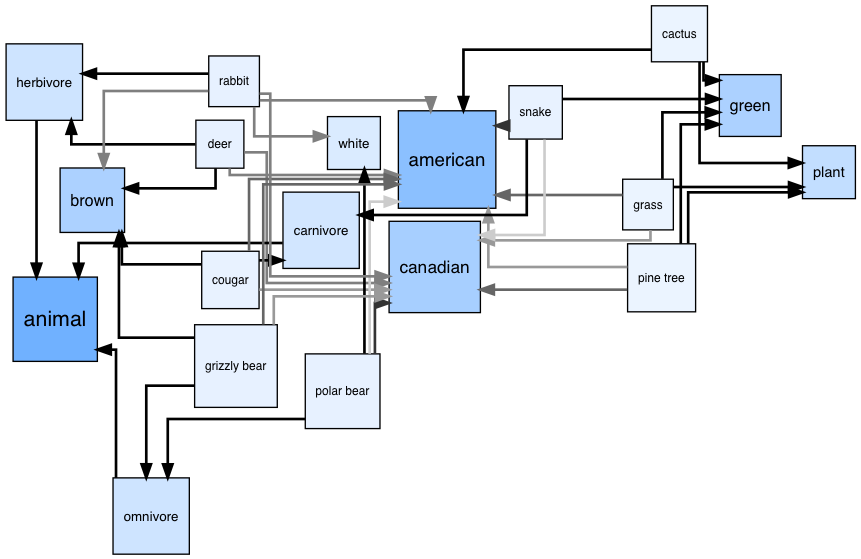}}
\subcaptionbox{Ground truth CPD}{\includegraphics[width=17em]{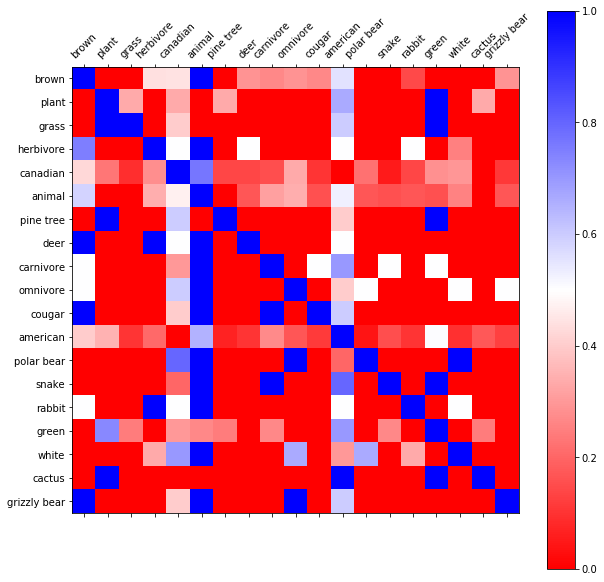}}
\caption{Representation of the toy probabilistic lattice used in Section \ref{warmup}. Darker color corresponds to more unary marginal probability. The associated CPD is obtained by a weighted aggregation of leaf elements.}
\label{orig-lattice-cpd-scb}
\end{figure*}
}
\begin{table*}[!htbp]
\addtocounter{table}{-1}
\refstepcounter{figure}
\label{toy-cpd}
\captionsetup{labelformat=empty}
\begin{tabular}{cc}
\begin{subfigure}{20em}\centering\includegraphics[width=16em]{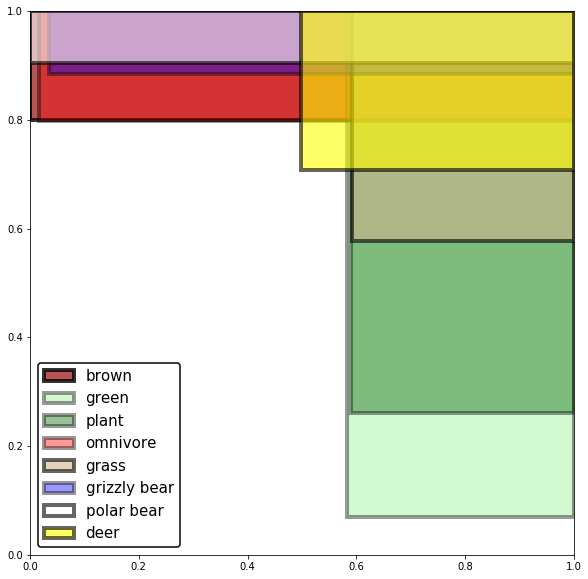}\caption{POE lattice}\label{poe-lattice}\end{subfigure}&
\begin{subfigure}{20em}\centering\includegraphics[width=16em]{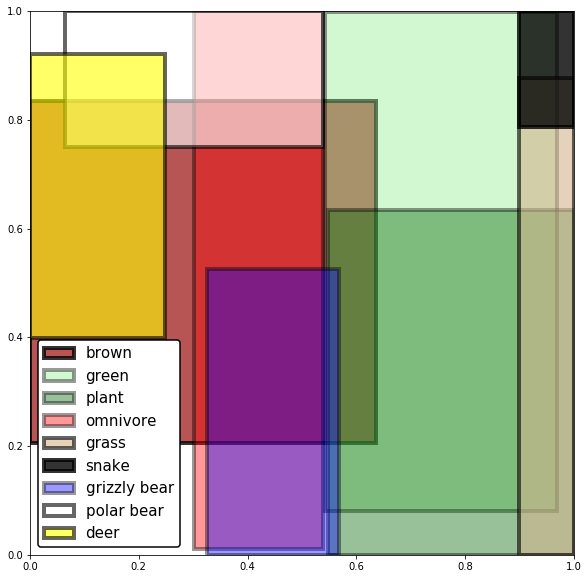}\caption{Box lattice}\label{cube-lattice}\end{subfigure}\\
\begin{subfigure}{20em}\centering\includegraphics[width=18em]{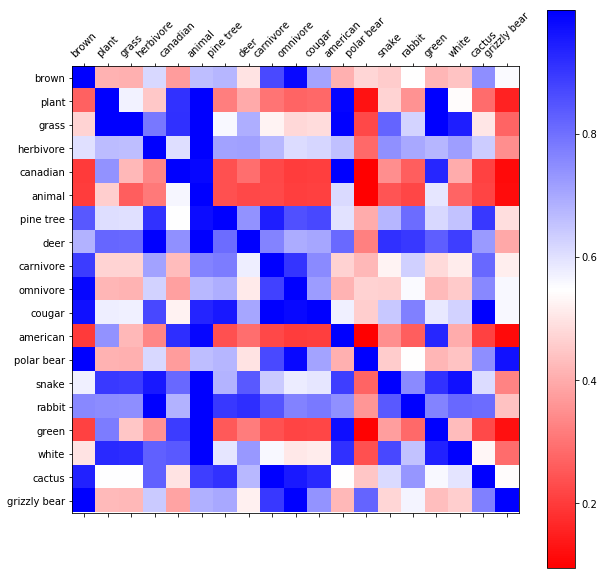}\caption{POE CPD}\label{poe-cpd}\end{subfigure}&
\begin{subfigure}{20em}\centering\includegraphics[width=18em]{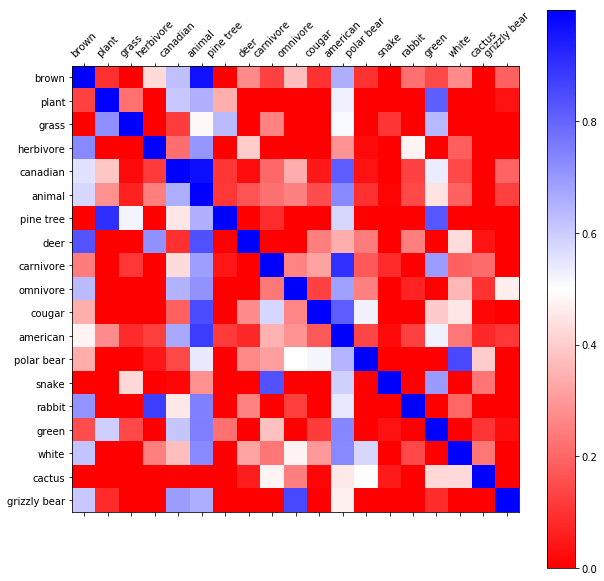}\caption{Box CPD}\label{cube-cpd}\end{subfigure}\\
\end{tabular}
\caption{Figure \thefigure: Lattice representations and conditional probabilities from POE vs. box lattice. Note how the box lattice model's lack of ``anchoring'' to a corner allows it vastly more expressivity in matching the ground truth CPD seen in Figure \ref{orig-lattice-cpd-scb}.}
\end{table*}
}
\begin{table*}[!ht]
\centering
\begin{tabular}{ll|ll|ll}
P(grizzly bear $|$ ... ) &                        & P(cactus $|$ ... )     &                        & P(plant $|$ ... ) &      \\ \hline
P(grizzly bear)          & 0.12                   & P(cactus)              & 0.10                   & P(plant)          & 0.20 \\
omnivore                 & \textcolor{blue}{0.29} & green                  & \textcolor{blue}{0.16} & green             & \textcolor{blue}{0.37} \\
white                    & \textcolor{red}{0.00}  & plant                  & \textcolor{blue}{0.39} & snake             & \textcolor{red}{0.00} \\
brown                    & \textcolor{blue}{0.30} & american, green        & \textcolor{blue}{0.19} & carnivore         & \textcolor{red}{0.00} \\
omnivore, white          & \textcolor{red}{0.00}  & plant, green, american & \textcolor{blue}{0.40} & cactus            & \textcolor{blue}{0.78} \\
omnivore, brown          & \textcolor{blue}{0.38} & american, carnivore    & \textcolor{red}{0.00}  & american, cactus  & \textcolor{blue}{0.85}
\end{tabular}
\caption{Multi-way queries: conditional probabilities adjust when adding additional evidence or contradiction. In constrast, POE can only raise or preserve probability when conditioning.}
\label{multi-query}
\end{table*}

We begin by investigating properties of our model in modeling a small toy problem, consisting of a small hand constructed ontology over 19 concepts, aggregated from atomic synthetic examples first into a probabilistic lattice (e.g. some rabbits are brown, some are white), and then a full CPD. We model it using only 2 dimensions to enable visualization of the way the model self-organizes its ``event space", training the model by minimize weighted cross-entropy with both the unary marginals and pairwise conditional probabilities. We also conduct a parallel experiment with POE as embedded in the unit cube, where each representation is constrained to touch the faces $x=1,y=1$. In Figure \ref{toy-cpd}, we show the representation of lattice structures  by POE and the box lattice model as compared to the abstract probabilistic lattice used to construct the data, shown in Figure \ref{orig-lattice-cpd-scb}, and compare the conditional probabilities produced by our model to the ground truth, demonstrating the richer capacity of the box model in capturing strong positive and negative correlations. In Table \ref{multi-query}, we perform a series of multivariable conditional queries and demonstrate intuitive results on high-order queries containing up to 4 variables, despite the model being trained on only 2-way information. 

\subsection{WordNet}

\begin{table}[!ht]
\centering
\begin{tabular}{l|l}
term1                   & term2                   \\ \hline
craftsman.n.02          & shark.n.03              \\
homogenized\_milk.n.01  & apple\_juice.n.01       \\
tongue\_depresser.n.01  & paintbrush.n.01 \\
deerstalker.n.01        & bathing\_cap.n.01       \\
skywriting.n.01         & transcript.n.01         \\
\end{tabular}
\caption{Negatively correlated variables produced by the model.}
\label{wordnet_neg}
\end{table}

\begin{table}[!ht]
\centering
\begin{tabular}{lc}
Method                        & Test Accuracy \% \\ \hline
transitive            & 88.2             \\
word2gauss                    & 86.6             \\
OE                            & 90.6             \\
\citet{li2017improved}        & 91.3             \\
DOE (KL)    & \textbf{92.3}    \\ \hline
POE                           & 91.6             \\
POE (100 dim)                 & 91.7             \\
Box                           & 92.2             \\
Box + CPD                     & \textbf{92.3}    \\
\end{tabular}
\caption{ Classification accuracy on WordNet test set.}
\label{wordnet-results}
\end{table}

We experiment on WordNet hypernym prediction, using the same train, development and test split as \citet{orderembedding}, created by randomly taking 4,000 hypernym pairs from the 837,888-edge transitive closure of the WordNet hypernym hierarchy as positive training examples for the development set, 4,000 for the test set, and using the rest as training data. Negative training examples are created by randomly corrupting a train/development/test edge $(u,v)$ by replacing either $u$ or $v$ with a randomly chosen negative node. We use their specific train/dev/test split, while \citet{athiwaratkun2018on} use a different train/dev split with the same test set (personal communication) to examine the effect of different negative sampling techniques. We cite their best performing model, called DOE (KL).

Since our model is probabilistic, we would like a sensible value for $P(n)$, where $n$ is a node. We assign these marginal probabilities by looking at the number of descendants in the hierarchy under a node, and normalizing over all nodes, taking $P(n) = \frac{|~\textit{descendants(n)}~|}{|~\textit{nodes}~|}$. 

Furthermore, we use the graph structure (only of the subset of edges in the training set to avoid leaking data) to augment the data with approximate conditional probabilities $P(x|y)$. For each leaf, we consider all of its ancestors as pairwise co-occurences, then aggregate and divide by the number of leaves to get an approximate joint probability distribution, $P(x, y) = \frac{|~\textit{x, y co-occur in ancestor set}~|}{|~\textit{leaves}~|}$. With this and the unary marginals, we can create a conditional probability table, which we prune based on the difference of $P(x|y)$ and $P(y|x)$ and add cross entropy with these conditional ``soft edges'' to the training data. We refer to experiments using this additional data as Box + CPD in Table \ref{wordnet-results}.

We use 50 dimensions in our experiments. Since our model has 2 parameters per dimension, we also perform an apples-to-apples comparison with a 100D POE model. As seen in Table \ref{wordnet-results}, we outperform POE significantly even with this added representational power. We also observe sensible negatively correlated examples, shown in \ref{wordnet_neg}, in the trained box model, while POE cannot represent such relationships. We tune our models on the development set, with parameters documented in Appendix \ref{wordnet-params}. We observe that not only does our model outperform POE, it beats all previous results on WordNet, aside from the concurrent work of \citet{athiwaratkun2018on} (using different train/dev negative examples), the baseline POE model does as well. This indicates that probabilistic embeddings for transitive relations are a promising avenue for future work. Additionally, the ability of the model to learn from the expected "soft edges" improves it to state-of-the-art level. We expect that co-occurrence counts gathered from real textual corpora, rather than merely aggregating up the WordNet lattice, would further strengthen this effect.

\subsection{Flickr Entailment Graph}

\begin{figure}[h]
\begin{center}
\includegraphics[width=8cm, height=5cm ]{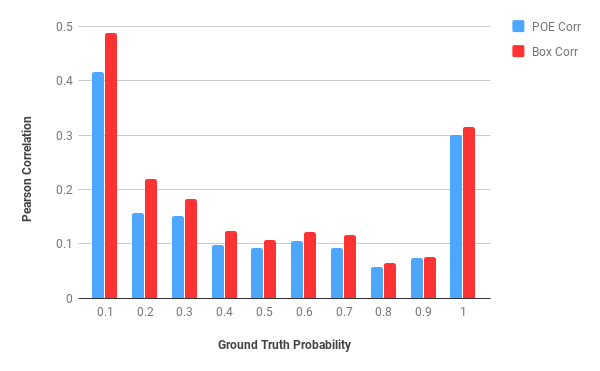}
\end{center}
\caption{R between model and gold probabilities.}
\label{corr_pic}
\end{figure}

\begin{table}[!ht]
\centering
\begin{tabular}{l|ll}
\hline
                        & $P(x|y)$         &                     \\ \cline{1-1}
\textit{Full test data} & KL  & Pearson R \\ \hline
POE                     & 0.031          & 0.949               \\
POE*                    & 0.031          & 0.949               \\
Box                     & \textbf{0.020} & \textbf{0.967}      \\ \hline
\textit{Unseen pairs}   &                &                     \\ \hline
POE                     & 0.048          & 0.920               \\
POE*                    & 0.046          & 0.925               \\
Box                     & \textbf{0.025} & \textbf{0.957}      \\ \hline
\textit{Unseen words}   &                &                     \\ \hline
POE                     & 0.127          & 0.696               \\
POE*                    & 0.084          & 0.854               \\
Box                     & \textbf{0.050} & \textbf{0.900}      \\ \hline
\end{tabular}
\caption{KL and Pearson correlation between model and gold probability.}
\label{flickr_table}
\end{table}

We conduct experiments on the large-scale Flickr entailment dataset of 45 million image caption pairs. We use the exactly same train/dev/test from \citet{lai2017learning}. We use a slightly different unseen word pairs and unseen words test data, obtained from the author. We include their published results and also use their published code, marked $*$, for comparison.

For these experiments, we relax our boxes from the unit hypercube to the nonnegative orthant and obtain probabilities under the exponential measure, $p(x) = \exp(-x)$. We enforce the nonnegativity constraints by clipping the LSTM-generated embedding \cite{hochreiter1997long} for the box minimum with a ReLU, and parametrize our $\Delta$ embeddings using a softplus activation to prevent dead units. As in \citet{lai2017learning}, we use 512 hidden units in our LSTM to compose sentence vectors. We then apply two single-layer feed-forward networks with 512 units applied to the final LSTM state to produce the embeddings.

As we can see from Table \ref{flickr_table}, we note large improvements in KL and Pearson correlation to the ground truth entailment probabilities. In further analysis, Figure \ref{corr_pic} demonstrates that while the box model outperforms POE in nearly every regime, the highest gains come from the comparatively difficult to calibrate small entailment probabilities, indicating the greater capability of our model to produce fine-grained distinctions.

\section{Conclusion and Future Work}

We have only scratched the surface of possible applications. An exciting direction is the incorporation of multi-relational data for general knowledge representation and inference. Secondly, more complex representations, such as $2n$-dimensional products of 2-dimensional convex polyhedra, would offer greater flexibility in tiling event space. Improved inference of the latent boxes, either through better optimization or through Bayesian approaches is another natural extension. Our greatest interest is in the application of this powerful new tool to the many areas where other structured embeddings have shown promise.

\section{Acknowledgments}
We thank Alice Lai for making the code from her original paper public, and for providing the additional unseen pairs and unseen words data. We also thank Haw-Shiuan Chang, Laurent Dinh, and Ben Poole for helpful discussions. We also thank the anonymous reviewers for their constructive feedback.

This work was supported in part by the Center for Intelligent Information Retrieval and the Center for Data Science, in part by the Chan Zuckerberg Initiative under the project Scientific Knowledge Base Construction., and in part by the National Science Foundation under Grant No. IIS-1514053. Any opinions, findings and conclusions or recommendations expressed in this material are those of the authors and do not necessarily reflect those of the sponsor.

\bibliography{paper}
\bibliographystyle{acl_natbib}

\newpage
\onecolumn
\appendix
\begin{center}
\Large{Supplementary Material}
\end{center}

\section{Queries with Negated Variables}
\label{negated-variables}

Section \ref{sec:box-lattice} mentions that although the complement of a box is not a box, queries involving negated variables can be calculated exactly with Inclusion-Exclusion, demonstrated in Table \ref{neg-query}. While there are many more interesting and efficient approaches, we simply use the formula for calculating the volume of the union of hyperrectangles (a standard Inclusion-Exclusion formula). 

This is equivalent since the intersection of complements of boxes is the complement of the union of boxes.  We first intersect all of the non-negated variables into one conjunction box, $T$. We then calculate the volume of the union of $T$ with all of the boxes representing complements of negated variables $F={\neg f_1, \neg f_2, \neg f_3,...}$, $v_1 = (T \cup f_1 \cup  f_2 \cup f_3 ...)=1- P(\neg T, \neg f_1,\neg f_2, \neg f_3, ...)$, and the volume of just the negated variables' boxes, $v_2 = (f_1 \cup  f_2 \cup f_3 ...) = 1-P(\neg f_1,\neg f_2, \neg f_3, ...)$. The probability of the query is $v_1 - v_2 = P(F) - P(\neg T, F)=(P(T,F) + P(\neg T, F)) - P(\neg T, F) = P(T,F)$, which was the original query.

\begin{table}[!ht]
\centering
\begin{tabular}{ll}
P(deer $|$ ... ) &                          \\ \hline
P(deer)          & 0.12                     \\
$\neg$white                 & \textcolor{blue}{0.13}   \\
animal                    & \textcolor{blue}{0.50}    \\
$\neg$white,animal                   & \textcolor{blue}{0.54}   \\
$\neg$white,animal,herbivore          & \textcolor{blue}{0.73}    \\
$\neg$white, animal, herbivore, $\neg$rabbit & \textcolor{blue}{0.80} \\
$\neg$white, animal, $\neg$herbivore,$\neg$rabbit          & \textcolor{red}{0.00}  
\end{tabular}
\caption{Negated variables: queries on the toy data with negated variables, calculated with Inclusion-Exclusion.}
\label{neg-query}
\end{table}

\section{Properties of the Box Lattice}
\label{lattice-properties}

In this section, we cover some technical details about the box lattice model and its properties especially as compared to the order embedding model.

\subsection{Non-Distributivity}

A lattice is called \emph{distributive} if the following identity holds for all members $x,y,z$:
\begin{align*}
x \wedge (y \vee z) = (x \wedge y) \vee (x \wedge z)
\end{align*}

\claim{Order embeddings form a distributive lattice.}
\proof{This is a standard results on vector lattices shown in e.g. \cite{rieszspace}}

A non-distributive lattice is a strictly more general object, capable of modeling more objects since it does not necessarily need to fulfill the above identity for all triples $x,y,z$.

\claim{The box lattice is non-distributive.}
\proof{Consider the box lattice in 1-dimension. Let $x=[0,0.3]$, $y=[0.2,0.6]$, and $z=[0.5,1.0]$. Then $x \wedge (y \vee z)=[0.2,0.3]$, but $(x \wedge y) \vee (x \wedge z) = [0,0.6] \vee \bot = [0,0.6]$.}

This proves that the box lattice is a strict generalization of order embeddings, and not equivalent to order embeddings of any dimensionality. Additionally, our choice of an example containing disjoint elements hints at the importance of non-distributivity for our goal of modeling disjoint events.

\subsection{Pseudocomplemented}

A lattice is called \emph{pseudocomplemented} if for every element $x$ there exists a unique greatest element in the lattice $x^*$ that is disjoint from $x$ and $x \wedge x^* = \bot$. The box lattice is almost always pseudocomplemented, aside from symmetry concerns (for example, a perfectly centered cube in the 2-dimensional box lattice of side length $<1$ has 4 possible equally large pseudocomplements. However any such symmetries can always be infinitesimally perturbed without breaking order structure so the box lattice is pseudocomplemented in a measure-theoretic sense. However, these pseudocomplements can be arbitrarily bad approximations of the true complement set of a box, with the worst case scenario coming from large, nearly-centered cubes.

\section{Asymmetrizing Score Matrices}
\label{asym-section}
\subsection{Probabilistic Models}
\label{prob-asym-proof}
Assume we have a pairwise CPD between Bernoulli variables, and also have access to the unary marginals for each Bernoulli, and further that no unary marginals are exactly identical. If they are exactly identical, we can generate random independent Bernoulli parameters and their JPD, and take a small convex combination with that to infinitesimally perturb the statistics, so this proof is valid everywhere but on a set of measure 0 which we can approximate arbitrarily well.

\claim{If all unary marginals are distinct, taking the elements of the pairwise CPD, removing the diagonal, and deleting an entry if $P(A|B) < P(B|A)$, that is if $A_{ij}<A_{ji}$, will result in a weighted adjacency matrix for an acyclic directed graph}
\proof{Order the variables $x_1...x_n$ so that $p(x_i)<p(x_j)$ if $i<j$. Now an entry of the CPD $p(x_i|x_j)=p(x_i,x_j)/p(x_j)=C_{ij}$ is less than $C_{ji}=p(x_i,x_j)/p(x_i)$ if $p(x_i)<p(x_j)$. So with the variables so ordered, if we use the CPD to create an adjacency matrix with an edge $C_{ij}=1$ if and only if $p(x_i)<p(x_j)$, it will be upper triangular with 0 on the diagonal. This is a nilpotent matrix which means it is the adjacency matrix of an acyclic graph. This can be easily seen since the entries of $A^k$ are the set of $K$-hop neighbors, and if this set eventually becomes empty, as in a nilpotent matrix, we have no cycles.}

Since the labeling of our vertices is arbitrary, this means that our adjacency matrix created by the proposed asymmetrizing procedure is always acyclic since it is similar to an upper triangular matrix with 0s on the diagonal. 

This holds as long as the unary marginals can always be ordered (which they can be except on a set of measure zero, and in practice on it seems to work even if you ignore this constraint.

\subsection{KL Divergences and Gaussian Embeddings}
\label{kl-gaussian-proof}

Assume the same setup as section \ref{prob-asym-proof}, but the scores in the matrix come from (possibly thresholded if $A_{ij}-A_{ji} < c$) pairwise divergences between Gaussian embeddings.

\claim{There exist graphs produced by the above procedure that do not lead to directed acyclic graphs if thresholded by deleting entries when $A_{ij} < A_{ji}$:}
\proof{Consider the following set of 5 2-dimensional Gaussians with diagonal covariance:
\begin{align*}
G_1&=\mathcal{N}(x_1;[-5, -3], \diag{([3, 7])})\\
G_2&=\mathcal{N}(x_2;[-3, 5], \diag{[(7, 4])})\\
G_3&=\mathcal{N}(x_3;[-5, -6], \diag{([8, 1])})\\
G_4&=\mathcal{N}(x_4;[-7, 6], \diag{([5, 5])})\\
G_5&=\mathcal{N}(x_5;[9, 3], \diag{([5, 9])})
\end{align*}
Applying asymmetrization and even pruning at a threshold of $c=1$ (which is non-nilpotent and does affect edges) produces a cycle between nodes 5, 1, and 3. There are certain repeated numbers in the parameters, but this is not the cause of the issue. They are whole numbers for ease of exposition, they were randomly generated and many more examples can be created with arbitrary floating point numbers.
}

\subsection{Order Embeddings}
\label{oe-conjecture}

We simulated many millions of random sets of order embedding parameters, and created pairwise graphs using the order embedding energy function, and were never able to find a cycle in the resulting asymmetrized graphs. We conjecture that this is because the order embedding energy is essentially a Lagrangian relaxation term penalizing the violation of a true partial order relation, but have not proven it.
\begin{conjecture}Sets of Order Embeddings can be consistently asymmetrized into directed acyclic graphs according to the procedure in section \ref{prob-asym-proof}.\end{conjecture}

\section{Model Parameters}
\subsection{WordNet Parameters}
\label{wordnet-params}

Since the WordNet data has binary ${0,1}$ links instead of calibrated probabilities, and the negative links are found from random negative sampling, we constrain the delta embedding to not update for negative samples during optimization. We found this was effective in preventing random negative samples from decreasing the volume of the boxes and creating artificially disjoint pairs. 

The WordNet parameters that achieved best performance on the development set (whose train set performance we reported) are: 
\begin{verbatim}
batch size: 800 
dimension: 50
edge loss weight: 1.0 
unary loss weight: 9.0
learning rate: 0.001
minimum dimension delta size: 1e-6
dimension-max regularization weight: 0.005
optimizer: Adam
\end{verbatim}

For WordNet training with additional soft CPD edges, we use the same parameters. We also perform pruning on the generated CPD file. We only include $\langle t_1, t_2 \rangle$ pairs with probability $\geq$ 0.6  and the reverse pair $\langle t_2, t_1 \rangle$ $\leq$ 0.4 probability.

We tune the batch size of the model between 800 and 40000 because bigger batch size facilitates faster training. We also sweep over 1.0 to 9.0 for edge loss weight and 9.0 to 1.0 for the unary loss weight. The learning rate we tune in $\lambda$ $\in$ \{0.001, 0.0001\}. The minimum dimension delta size we tune in $\in$ \{0.01, 0.001, 0.0001, 0.00001, 0.000001\}. The dimension-max regularization encourages the upper bound of box to be close 1.0 with an L1 penalty to prevent collapse. We perform parameter search in \{0.0, 0.001, 0.005, 0.01, 0.05, 0.1, 0.5\}.

\subsection{Flickr Parameters}
\label{flickr-params}

The Flickr parameters that achieved best performance on the development set (whose train set performance we reported) are: 
\begin{verbatim}
batch size: 512
dropout: 0.5
unary loss weight: 8.0
edge loss weight: 2.0 
learning rate: 0.0001
minimum dimension delta size: 1e-6
optimizer: Adam
\end{verbatim}
The LSTM parameters are initialized with Glorot initialization \cite{glorot2010understanding}, as are the weight and bias parameters for the feedforward networks to produce the box minimums. The network to produce the $\Delta$ embedding is initialized from a uniform distribution from $[15.0, 15.50]$. We clip to zero for min embeddings (apply a ReLU), and apply a softplus to enforce the positivity and minimum dimension size constraints on the $\Delta$ embeddings.

We also sweep over 1.0 to 9.0 for edge loss weight and 9.0 to 1.0 for the unary loss weight. The learning rate $\lambda$ $\in$ \{0.001, 0.0001\}. We tried Glorot initialization with the $\Delta$ network as well, but since we wanted a high degree of overlap at the beginning of training, we simply swept over different uniform initialization ranges in $[5.0, 5.5]$, $[10.0, 10.5]$ and $[15.0, 15.5]$.

\end{document}